\documentclass[letterpaper,11pt]{article}

\usepackage{jmlr2e}
\usepackage{mkolar_definitions}
\usepackage{thmtools,thm-restate}

\jmlrheading{}{}{}{}{}{}{}

\ShortHeadings{Exponential Weights on the Hypercube in Polynomial Time}{Putta and Shetty}

\begin{document}

\title{Exponential Weights on the Hypercube in Polynomial Time}

\author{\name Sudeep Raja Putta \email sp3794@columbia.edu \\
     \addr Columbia University 
     \AND  
      \name Abhishek Shetty \email avs88@cornell.edu \\
     \addr Cornell University}
       

\editor{}

\maketitle

\begin{abstract}
We study a general online linear optimization problem(OLO). At each round, a subset of objects from a fixed universe of $n$ objects is chosen, and a linear cost associated with the chosen subset is incurred. To measure the performance of our algorithms, we use the notion of regret which is the difference between the total cost incurred over all iterations and the cost of the best fixed subset in hindsight. We consider Full Information and Bandit feedback for this problem. This problem is equivalent to OLO on the $\{0,1\}^n$ hypercube. The Exp2 algorithm and its bandit variant are commonly used strategies for this problem. It was previously unknown if it is possible to run Exp2 on the hypercube in polynomial time.

In this paper, we present a polynomial time algorithm called PolyExp for OLO on the hypercube. We show that our algorithm is equivalent Exp2 on $\{0,1\}^n$, Online Mirror Descent(OMD), Follow The Regularized Leader(FTRL) and Follow The Perturbed Leader(FTPL) algorithms. We show PolyExp achieves expected regret bound that is a factor of $\sqrt{n}$ better than Exp2 in the full information setting under $L_\infty$ adversarial losses. Because of the equivalence of these algorithms, this implies an improvement on Exp2's regret bound in full information. We also show matching regret lower bounds. Finally, we show how to use PolyExp on the $\{-1,+1\}^n$ hypercube, solving an open problem in Bubeck et al (COLT 2012).
\end{abstract}

\section{Introduction}
Consider the following abstract game which proceeds as a sequence of $T$ rounds. In each round $t$, a player has to choose a subset $S_t$ from a universe $U$ of $n$ objects. Without loss of generality, assume $U=\{1,2,..,n\}=[n]$. Each object $i \in U$ has an associated loss $c_{t,i}$, which is unknown to the player and may be chosen by an adversary. On choosing $S_t$, the player incurs the cost $c_t(S_t) = \sum_{i \in S_t} c_{t,i}$. In addition the player receives some feedback about the costs of this round. The goal of the player is to choose the subsets such that the total cost incurred over a period of rounds is close to to the total cost of the best subset in hindsight. This difference in costs is called the \textit{regret} of the player. Formally, regret is defined as:
$$\mathcal{R}_T = \sum_{t=1}^T c_t(S_t) - \min_{S \subseteq U} \sum_{t=1}^T c_t(S)$$
We can re-formulate the problem as follows. The $2^n$ subsets of $U$ can be mapped to the vertices of the $\{0,1\}^n$ hypercube. The vertex corresponding to the set $S$ is represented by its characteristic vector $X(S) = \sum_{i=1}^n 1\{i \in S\} e_i$. From now on, we will work with the hypercube instead of sets and use losses $l_{t,i}$ instead of costs. In each round, the player chooses $X_t \in \{0,1\}^n$. The loss vector $l_{t}$ is be chosen by an adversary and is unknown to the player. The loss of choosing $X_t$ is $X_t^\top l_t$. The player receives some feedback about the loss vector. The goal is to minimize regret, which is now defined as:
$$\mathcal{R}_T = \sum_{t=1}^T X_t^\top l_t - \min_{X \in \{0,1\}^n}\sum_{t=1}^T X^\top l_t$$
This is the \textit{Online Linear Optimization(OLO)} problem on the hypercube. As the loss vector $l_t$ can be set by an adversary, the player has to use some randomization in its decision process in order to avoid being foiled by the adversary. At each round $t=1,2,\dots,T$, the player chooses an action $X_t$ from the decision set $\{0,1\}^n$, using some internal randomization. Simultaneously, the adversary chooses a loss vector $l_t$, without access to the internal randomization of the player. Since the player's strategy is randomized and the adversary could be adaptive, we consider the expected regret of the player as a measure of the player's performance. Here the expectation is with respect to the internal randomization of the player and the adversary's randomization. 

We consider two kinds of feedback for the player.
\begin{enumerate}
\item \textit{Full Information setting:} At the end of each round $t$, the player observes the loss vector $l_t$.
\item \textit{Bandit setting:} At the end of each round $t$, the player only observes the scalar loss incurred $X_t^\top l_t$.
\end{enumerate}
In order to make make quantifiable statements about the regret of the player, we need to restrict the loss vectors the adversary may choose. Here we assume that $\norm{l_t}_{\infty} \leq 1$ for all $t$, also known as the $L_\infty$ assumption.

There are three major strategies for online optimization, which can be tailored to the problem structure and type of feedback. Although, these can be shown to be equivalent to each other in some form, not all of them may be efficiently implementable. These strategies are:
\begin{enumerate}
\item Exponential Weights (EW)\citep{freund1997decision,littlestone1994weighted} 
\item Follow The Leader (FTL)\citep{kalai2005efficient}
\item Online Mirror Descent (OMD) \citep{nemirovsky1983problem}.
\end{enumerate}

For problems of this nature, a commonly used EW type algorithm is Exp2 \citep{audibert2011minimax, audibert2013regret, bubeck2012towards}. For the specific problem of Online Linear Optimization on the hypercube, it was previously unknown if the Exp2 algorithm can be efficiently implemented \citep{bubeck2012towards}. So, previous works have resorted to using OMD algorithms for problems of this kind. The main reason for this is that Exp2 explicitly maintains a probability distribution on the decision set. In our case, the size of the decision set is $2^n$. So a straightforward implementation of Exp2 would need exponential time and space. 

\subsection{Our Contributions}
We use the following key observation: In the case of linear losses the probability distribution of Exp2 can be factorized as a product of $n$ Bernoulli distributions. Using this fact, we design an efficient polynomial time algorithm called \textit{PolyExp} for sampling sampling from and updating these distributions.

We show that PolyExp is equivalent to Exp2. In addition, we show that PolyExp is equivalent to OMD with entropic regularization and Bernoulli sampling. This allows us to analyze PolyExp's using powerful analysis techniques of OMD. We also show that PolyExp is equivalent to a Follow The Regularized Leader(FTRL) and a Follow The Perturbed Leader(FTPL) algorithms.

This kind of equivalence is rare. To the best of our knowledge, the only other scenario where this kind of equivalence holds is on the probability simplex for the so called experts problem.

In our paper, we focus on the $L_\infty$ assumption. In the full information setting, directly analyzing Exp2 gives a regret bound of $O(n^{3/2}\sqrt{T})$. Using the equivalence to OMD, we show that PolyExp's regret bound is $n\sqrt{T}$.  In the bandit setting, PolyExp's regret cannot be bounded through our analysis, when using the one point linear estimator proposed in \cite{dani2008price}. However, since we show that Exp2 and PolyExp are equivalent, they must have the same regret bound. These results are summarized by the table below.
\begin{center}
 \begin{tabular}{|c|c|c|}
 \hline
 \multicolumn{3}{|c|}{$L_\infty$}\\
 \hline
  & Full Information & Bandit \\ 
 \hline
 Exp2 & $O(n^{3/2} \sqrt{T})$ & $O(n^{2} \sqrt{T})$  \\ 
 \hline
 PolyExp &  $O(n \sqrt{T})$ & $O(n^{2} \sqrt{T})$  \\
 \hline
 Lowerbound &  $\Omega(n \sqrt{T})$ & $\Omega(n^{2} \sqrt{T})$  \\
 \hline
\end{tabular}
\end{center}
\begin{proposition}For the Online Linear Optimization problem on the $\{0,1\}^n$  Hypercube, Exp2, OMD, FTRL, FTPL and PolyExp are equivalent. Moreover, under $L_\infty$ adversarial losses, these algorithms have the following regret:
\begin{enumerate}
\item Full Information: $O(n\sqrt{T})$
\item Bandit: $O(n^{2} \sqrt{T})$.
\end{enumerate}
\end{proposition}
We also have the following lower bounds.
\begin{proposition}For the Online Linear Optimization problem on the $\{0,1\}^n$  Hypercube with $L_\infty$ adversarial losses, the regret of any algorithm is at least:
\begin{enumerate}
\item Full Information: $\Omega\left(n\sqrt{T}  \right)$
\item Bandit: $\Omega(n^{2} \sqrt{T})$.
\end{enumerate}
\end{proposition}
Finally, in \citep{bubeck2012towards}, the authors state that it is not known if it is possible to sample from the exponential weights distribution in polynomial time for $\{-1,+1\}^n$ hypercube. We show how to use PolyExp on $\{0,1\}^n$ for $\{-1,+1\}^n$. We show that the regret of such an algorithm on $\{-1,+1\}^n$ will be a constant factor away from the regret of the algorithm on $\{0,1\}^n$. Thus, we can use PolyExp to obtain a polynomial time algorithm for $\{-1,+1\}^n$ hypercube.

We present the proofs of equivalence and regret of PolyExp within the main body of the paper. The remaining proofs are deferred to the appendix.

\subsection{Relation to Previous Works}
In previous works on OLO \citep{dani2008price, koolen2010hedging, audibert2011minimax, cesa2012combinatorial, bubeck2012towards, audibert2013regret} the authors consider arbitrary subsets of $\{0,1\}^n$ as their decision set. This is also called as Online Combinatorial optimization. In our work, the decision set is the entire $\{0,1\}^n$ hypercube. Moreover, the assumption on the adversarial losses are different. Most of the previous works use the $L_2$ assumption \citep{bubeck2012towards,dani2008price, cesa2012combinatorial} and some use the $L_\infty$ assumption \citep{koolen2010hedging,audibert2011minimax}.

The Exp2 algorithm has been studied under various names, each with their own modifications and improvements. In its most basic form, it corresponds to the Hedge algorithm from \citep{freund1997decision} for full information. For combinatorial decision sets, it has been studied by \citep{koolen2010hedging} for full information. In the bandit case, several variants of Exp2 exist based on the sampling scheme and linear estimator used. These were studied in \citep{dani2008price, cesa2012combinatorial} and \citep{bubeck2012towards}. It has been proven in \citep{audibert2011minimax} that Exp2 is provably sub optimal for some decision sets and losses.

Follow the Leader kind of algorithms were introduced by \citep{kalai2005efficient} for the full information setting, which can be extended to the bandit settings in some cases.

Mirror descent style of algorithms were introduced in \citep{nemirovsky1983problem}. For online learning, several works \citep{aber2009compe, koolen2010hedging, bubeck2012towards, audibert2013regret} consider OMD style of algorithms. Other algorithms such as Hedge, FTRL and FTPL etc can be shown to be equivalent to OMD with the right regularization function and purturbation distribution. In fact, \citep{srebro2011universality} show that OMD can always achieve a nearly optimal regret guarantee for a general class of online learning problems.

Under the $L_\infty$ assumption, \citep{koolen2010hedging} \citep{audibert2011minimax} and \citep{cohen2017tight}  present lower bounds that match our lower bounds. However, they prove that there exists a subset $S \subset \{0,1\}^n$ and a sequence of losses on $S$ such that the regret is at least some lower bound. So, these results are not directly applicable in our case. So, we derive lower bounds specific for the entire hypercube, showing that there exists a sequence of losses on $\{0,1\}^n$ such that the regret is at least some lower bound. 

We refer the readers to the books by \citep{cesa2006prediction}, \citep{bubeck2012regret}, \citep{shalev2012online}, \citep{hazan2016introduction} and lectures by \citep{rakhlin2009lecture}, \citep{bubeck2011introduction} for a comprehensive survey of online learning algorithms.
\section{Algorithms and Equivalences}
In this section, we describe and analyze the Exp2, OMD with Entropic regularization and Bernoulli Sampling, and PolyExp algorithms and prove their equivalence.
\subsection{Exp2}
\begin{figure}[h]
\noindent\fbox{%
    \parbox{\linewidth}{%
\textbf{Algorithm:} Exp2\\
\textbf{Parameters:} Learning Rate $\eta$\\
Let $w_1(X) = 1$ for all $X \in \{0,1\}^n$. For each round $t=1,2,\dots,T$:
\begin{enumerate}
\item Sample $X_t$ as below. Play $X_t$ and incur the loss $X_t^\top l_t$.
\begin{enumerate}
\item Full Information: $X_t \sim p_t(X) = \frac{w_t(X)}{Z_t}$, where $Z_t=\sum \limits_{Y \in \{0,1\}^n} w_t(Y)$
\item Bandit: $X_t \sim q_t(X) = (1-\gamma)p_t(X) + \gamma \mu(X)$. Here $\mu$ is the exploration distribution.
\end{enumerate}
\item See Feedback and construct $\tilde{l_t}$.
\begin{enumerate}
\item Full Information: $ \tilde{l_t} = l_t$.
\item Bandit: $\tilde{l}_{t} = P_t^{-1}X_t X_t^\top l_t$, where $P_t = \mathbb{E}_{X \sim q_t}[XX^\top]$
\end{enumerate}
\item Update for all $X \in \{0,1\}^n$
\begin{align*}
w_{t+1}(X) = \exp\left(-\eta \sum_{\tau=1}^tX^\top \tilde{l}_\tau \right) \quad \text{ or equivalently }\quad
 w_{t+1}(X) = \exp(-\eta X^\top \tilde{l}_t)w_t(X)
\end{align*}
\end{enumerate}
}%
}
\end{figure}
The loss vector used to update Exp2 must satisfy the condition that $\mathbb{E}_{X_t}[\tilde{l_t}] = l_t$. In the bandit case, the estimator was first proposed by \citep{dani2008price}. Here, $\mu$ is the exploration distribution and $\gamma$ is the mixing coefficient. We use uniform exploration over $\{0,1\}^n$.

Exp2 has several computational drawbacks. First, it uses $2^n$ parameters to maintain the distribution $p_t$. Sampling from this distribution in step 1 and updating it step 3 will require exponential time. For the bandit settings, even computing $\tilde{l_t}$ will require exponential time.
We state the following regret bounds by analyzing Exp2 directly. The proofs are in the appendix. Later, we prove that these can be improved. These regret bounds are under the $L_\infty$ assumption.

\begin{restatable}{theorem}{ExpFReg}  \label{Theorem1}In the full information setting, if $\eta = \sqrt{\frac{\log 2}{nT}}$, Exp2 attains the regret bound:
$$E[\mathcal{R}_T] \leq 2 n^{3/2}\sqrt{T\log 2}$$
\end{restatable}

\begin{restatable}{theorem}{ExpBanReg}
In the bandit setting, if $\eta = \sqrt{\frac{\log 2}{9n^2T}}$ and $\gamma = 4n^2 \eta$, Exp2 with uniform exploration on $\{0,1\}^n$ attains the regret bound:
$$\mathbb{E}[\mathcal{R}_T] \leq 6 n^{2} \sqrt{T \log 2}$$
\end{restatable}

\subsection{PolyExp}
\begin{figure}[h]
\noindent\fbox{%
    \parbox{\linewidth}{%
\textbf{Algorithm:} PolyExp\\
\textbf{Parameters:} Learning Rate $\eta$\\
Let $x_{i,1} = 1/2$ for all $i \in [n]$. For each round $t=1,2,\dots,T$:
\begin{enumerate}
\item Sample $X_t$ as below. Play $X_t$ and incur the loss $X_t^\top l_t$.
\begin{enumerate}
\item Full information: $X_{i,t} \sim Bernoulli(x_{i,t})$
\item Bandit: With probability $1-\gamma$ sample $X_{i,t} \sim Bernoulli(x_{i,t})$ and with probability $\gamma$ sample $X_t \sim \mu$ 
\end{enumerate}
\item See Feedback and construct $\tilde{l}_t$
\begin{enumerate}
\item Full information: $\tilde{l}_t  = l_t$
\item Bandit: $\tilde{l}_{t} = P_t^{-1}X_t X_t^\top l_t$, where $P_t = (1-\gamma)\Sigma_t + \gamma \mathbb{E}_{X \sim \mu}[XX^\top]$. The matrix $\Sigma_t$ is $\Sigma_t[i,j]  = x_{i,t}x_{j,t}$ if $i\neq j$ and $\Sigma_t[i,i] = x_i$ for all $i,j \in [n]$
\end{enumerate}
\item Update for all $i \in [n]$:
\begin{align*}
x_{i,t+1} &= \frac{1}{1+\exp(\eta \sum_{\tau=1}^t \tilde{l}_{i,\tau})} \text{ or equivalently}\\
x_{i,t+1} &= \frac{x_{i,t}}{x_{i,t} + (1-x_{i,t})\exp(\eta \tilde{l}_{i,t})}
\end{align*}
\end{enumerate}
    }%
}
\end{figure}

To get a polynomial time algorithm, we replace the sampling and update steps with polynomial time operations. PolyExp uses $n$ parameters represented by the vector $x_t$. Each element of $x_t$ corresponds to the mean of a Bernoulli distribution. It uses the product of these Bernoulli distributions to sample $X_t$ and uses the update equation mentioned in step 3 to obtain $x_{t+1}$.

In the Bandit setting, we can sample $X_t$ by sampling from $\prod_{i=1}^n Bernoulli(x_{t,i})$ with probability $1-\gamma$ and sampling from $\mu$ with probability $\gamma$. As we use the uniform distribution over $\{0,1\}^n$ for exploration, this is equivalent to sampling from $\prod_{i=1}^n Bernoulli(1/2)$. So we can sample from $\mu$ in polynomial time. The matrix $P_t = \mathbb{E}_{X\sim q_t}[XX^\top] = (1-\gamma)\Sigma_t +\gamma \Sigma_\mu$. Here $\Sigma_t$ and $\Sigma_\mu$ are the covariance matrices when $X \sim \prod_{i=1}^n Bernoulli(x_{t,i})$ and $X \sim \prod_{i=1}^n Bernoulli(1/2)$ respectively. It can be verified that $\Sigma_t[i,j]  = x_{i,t}x_{j,t}, \Sigma_\mu[i,j]  = 1/4$ if $i\neq j$ and $\Sigma_t[i,i] = x_i, \Sigma_\mu[i,i]  = 1/2$ for all $i,j \in [n]$. So $P_t^{-1}$ can be computed in polynomial time.
\subsection{Equivalence of Exp2 and PolyExp}
We prove that running Exp2 is equivalent to running PolyExp.

\begin{restatable}{theorem}{Equiv}Under linear losses $\tilde{l}_t$, Exp2 on $\{0,1\}^n$ is equivalent to PolyExp. At round $t$, The probability that PolyExp chooses $X$ is $\prod_{i=1}^n (x_{i,t})^{X_i} (1-x_{i,t})^{(1-X_i)}$ where $x_{i,t} = (1+\exp(\eta \sum_{\tau=1}^{t-1} \tilde{l}_{i,\tau}))^{-1}$. This is equal to the probability of Exp2 choosing $X$ at round $t$, ie:
$$\prod_{i=1}^n (x_{i,t})^{X_i} (1-x_{i,t})^{(1-X_i)} = \frac{\exp(-\eta \sum_{\tau=1}^{t-1}X^\top \tilde{l}_\tau)}{Z_t}$$
where $Z_t = \sum_{Y \in \{0,1\}^n}\exp(-\eta\sum_{\tau=1}^{t-1} Y^\top \tilde{l}_\tau)$.
\end{restatable}

At every round, the probability distribution $p_t$ in Exp2  is the same as the product of Bernoulli distributions in PolyExp. Lemma \ref{Lemma2} is crucial in proving equivalence between the two algorithms. In a strict sense, Lemma \ref{Lemma2} holds only because our decision set is the entire $\{0,1\}^n$ hypercube. The vector $\tilde{l}_t$ computed by Exp2 and PolyExp will be same. Hence, Exp2 and PolyExp are equivalent. Note that this equivalence is true for any sequence of losses as long as they are linear.
\subsection{Online Mirror Descent}
We present the OMD algorithm for linear losses on general finite decision sets. Our exposition is adapted from \citep{bubeck2012regret} and \citep{shalev2012online}. Let $\mathcal{X} \subset \mathbb{R}^n$ be an open convex set and $\mathcal{\bar{X}}$ be the closure of $\mathcal{X}$. Let $\mathcal{K} \in \mathbb{R}^d$ be a finite decision set such that $\mathcal{\bar{X}}$ is the convex hull of $\mathcal{K}$. The following definitions will be useful in presenting the algorithm.
\begin{definition}\textbf{Legendre Function:} A continuous function $R: \mathcal{\bar{X}} \to \mathbb{R}$ is Legendre if
\begin{enumerate}
\item $R$ is strictly convex and has continuous partial derivatives on $\mathcal{X}$.
\item  $\lim \limits_{x \to \mathcal{\bar{X}}/\mathcal{X}} \|\nabla R(x)\| = +\infty$
\end{enumerate}
\end{definition}

\begin{definition}\textbf{Legendre-Fenchel Conjugate:} Let $R:\mathcal{\bar{X}} \to \mathbb{R}$ be a Legendre function. The Legendre-Fenchel conjugate of $R$ is:
$$R^\star(\theta) = \sup_{x \in \mathcal{X}} (x^\top \theta - R(x))$$
\end{definition}

\begin{definition} \textbf{Bregman Divergence:} Let $R(x)$ be a Legendre function, the Bregman divergence $D_R:\mathcal{\bar{X}} \times \mathcal{X} \to \mathbb{R} $ is:
$$D_R(x\|y) = R(x) - R(y) - \nabla R(y)^\top (x-y)$$
\end{definition}

\begin{figure}[h]
\noindent\fbox{%
    \parbox{\linewidth}{%
\textbf{Algorithm:} Online Mirror Descent with Regularization $R(x)$\\
\textbf{Parameters:} Learning Rate $\eta$\\
Pick $x_1 = \arg \min \limits_{x \in  \mathcal{\bar{X}}} R(x)$. For each round $t=1,2,\dots,T$:
\begin{enumerate}
\item Let $p_t$ be a distribution on $\mathcal{K}$ such that $\mathbb{E}_{X \sim p_t}[X] = x_t$. Sample $X_{t}$ as below and incur the loss $X_t^\top l_t$
\begin{enumerate}
\item Full information: $X_{t} \sim p_t$
\item Bandit: With probability $1-\gamma$ sample $X_{t} \sim p_t$ and with probability $\gamma$ sample $X_t \sim \mu$.
\end{enumerate}
\item See Feedback and construct $\tilde{l}_t$
\begin{enumerate}
\item Full information: $\tilde{l}_t  = l_t$
\item Bandit: $\tilde{l}_{t} = P_t^{-1}X_t X_t^\top l_t$, where $P_t = (1-\gamma)\mathbb{E}_{X \sim p_t}[XX^\top] + \gamma \mathbb{E}_{X \sim \mu}[XX^\top]$. 
\end{enumerate}
\item Let $y_{t+1}$ satisfy: $y_{t+1} = \nabla R^\star(\nabla R(x_t)-\eta \tilde{l_t})$
\item Update $x_{t+1} = \arg\min_{x \in \mathcal{\bar{X}}}D_R(x||y_{t+1})$
\end{enumerate}
    }%
}
\end{figure}
\subsection{Equivalence of PolyExp  and Online Mirror Descent}
For our problem, $\mathcal{K} = \{0,1\}^n$, $\mathcal{\bar{X}} = [0,1]^n$ and $\mathcal{X} = (0,1)^n$. We use entropic regularization:$$R(x) = \sum_{i=1}^n x_i \log x_i + (1-x_i) \log (1-x_i)$$
This function is Legendre. The OMD algorithm does not specify the probability distribution $p_t$ that should be used for sampling. The only condition that needs to be met is $\mathbb{E}_{X \sim p_t}[X] = x_t$, i.e, $x_t$ should be expressed as a convex combination of $\{0,1\}^n$ and probability of picking $X$ is its coefficient in the linear decomposition of $x_t$. An easy way to achieve this is by using Bernoulli sampling like in PolyExp. Hence, we have the following equivalence theorem:
\begin{restatable}{theorem}{EqOMD}
Under linear losses $\tilde{l}_t$, OMD on $[0,1]^n$ with Entropic Regularization and Bernoulli Sampling is equivalent to PolyExp.
The sampling procedure of PolyExp satisfies $\mathbb{E}[X_t] = x_t$. The update of OMD with Entropic Regularization is the same as PolyExp. 
\end{restatable}
In the bandit case, if we use Bernoulli sampling, $\mathbb{E}_{X \sim p_t}[XX^\top] = \Sigma_t$.
\subsection{Regret of PolyExp via OMD analysis}
Since OMD and PolyExp are equivalent, we can use the standard analysis tools of OMD to derive a regret bound for PolyExp. These regret bounds are under the $L_\infty$ assumption.
\begin{restatable}{theorem}{RegPoly}\label{Theorem2}
In the full information setting, if $\eta = \sqrt{\frac{\log 2}{T}}$, PolyExp attains the regret bound:
$$E[\mathcal{R}_T] \leq 2n\sqrt{T\log 2}$$
\end{restatable}

We have shown that Exp2 on $\{0,1\}^n$ with linear losses is equivalent to PolyExp. We have also shown that PolyExp's regret bounds are tighter than the regret bounds that we were able to derive for Exp2 in full information. This naturally implies an improvement for Exp2's regret bounds as it is equivalent to PolyExp and must attain the same regret. However, in the bandit case PolyExp does not improve Exp2's regret bound. So, it has the same regret as Exp2 stated in Theorem 5.

\subsection{Follow The Leader}
PolyExp can be shown to be equivalent to a Follow The Regularized Leader(FTRL) and a Follow The Perturbed Leader(FTPL) algorithm. The FTRL algorithm can be easily deduced from the OMD algorithm as the Bregman projection step in OMD is not necessary in the case of entropic regularization. Hence, to derive the FTRL algorithm, we replace steps 3 and 4 in OMD with the following update step:
\begin{align*}
x_{t+1} &= \arg\min_{x \in [0,1]^n} \left[ \eta \sum _ { \tau = 1 } ^ { t } \tilde{l} _ { \tau } ^ { \top } x + R( x ) \right]
\end{align*}
In the FTPL algorithm, we draw a random vector $v$ from some fixed distribution at each time step. We then choose the point which minimizes the sum of all previous losses and the perturbation vector $v$, by solving a simple linear optimization problem. 
$$X_{t+1} = \arg\min_{x \in \{0,1\}^n} \left[ \eta \sum _ { \tau = 1 } ^ { t } \tilde{l} _ { \tau } ^ { \top } x + v^\top x \right]$$
The FTPL algorithm is quite easy to implement in the full information setting as it folds the update and sampling steps into one computationally efficient step.

In the bandit setting however, we need to know the distribution $p_t(X)$, which is defined as follows:
$$p_t(X) = \Pr\left(X = \arg \min_{\{0,1\}^n}  \left[ \eta \sum _ { \tau = 1 } ^ { t } \tilde{l} _ { \tau } ^ { \top } x + v^\top x \right] \right)$$
This distribution may not always be computable.

\begin{figure}[h]
\noindent\fbox{%
    \parbox{\linewidth}{%
\textbf{Algorithm:} Follow the Perturbed Leader with Cumulative Distribution Function $F(x)$ \\
\textbf{Parameters:} Learning Rate $\eta$\\
Pick $X_1 = \arg \min \limits_{X \in  \{0,1\}^n} v^\top x$, where $v$ is drawn from distribution having CDF $F(x)$. For each round $t=1,2,\dots,T$:
\begin{enumerate}
\item \begin{enumerate}
\item Full information: Play $X_{t}$
\item Bandit: With probability $1-\gamma$ play $X_{t}$ and with probability $\gamma$ play $X_t \sim \mu$.
\end{enumerate}
\item See Feedback and construct $\tilde{l}_t$
\begin{enumerate}
\item Full information: $\tilde{l}_t  = l_t$
\item Bandit: $\tilde{l}_{t} = P_t^{-1}X_t X_t^\top l_t$, where $P_t = (1-\gamma)\mathbb{E}_{X \sim p_t}[XX^\top] + \gamma \mathbb{E}_{X \sim \mu}[XX^\top]$.
\end{enumerate}
\item Update $$X_{t+1} = \arg\min_{x \in \{0,1\}^n} \left[ \eta \sum _ { \tau = 1 } ^ { t } \tilde{l} _ { \tau } ^ { \top } x + v^\top x \right]$$
\end{enumerate}
    }%
}
\end{figure}
We choose $F(x)$ to be the product of logistic distributions, ie distributions with CDF $\Pr(x\leq \theta) = (1+\exp(-\theta))^{-1}$. In this case, we show that this FTPL is equivalent to PolyExp. Also, the distribution $p_t(X)$ has a simple closed form and $\mathbb{E}_{X \sim p_t}[XX^\top]  = \Sigma_t$.

\begin{restatable}{theorem}{EqFTPRL}
Under linear losses $\tilde{l}_t$, FTRL with Entropic regularization and Bernoulli sampling, and FTPL  with iid Logistic perturbations are equivalent to PolyExp.
\end{restatable}

\section{Comparison of Exp2's and PolyExp's regret proofs}
Consider the results we have shown so far. We proved  that PolyExp and Exp2 on the hypercube are equivalent. So logically, they should have the same regret bounds. But, our proofs say that PolyExp's regret is $O(\sqrt{n})$ better than Exp2's regret. What is the reason for this apparent discrepancy?

The answer lies in the choice of $\eta$ and the application of the inequality $e^{-x} \leq 1+x-x^2$ in our proofs. This inequality is valid when $x \geq -1$. When analyzing Exp2, $x$ is $\eta X^\top l_t = \eta L_t(X)$. So, to satisfy the constraints $x \geq -1$ we enforce that $|\eta L_t(X)| \leq 1$. Since $|L_t(X)|\leq n$, $\eta \leq 1/n$. When analyzing PolyExp, $x$ is $\eta l_{t,i}$ and we enforce that $|\eta l_{t,i}| \leq 1$. Since we already assume $|l_{t,i}| \leq 1$, we get that $\eta \leq 1$. PolyExp's proof technique allows us to find a better $\eta$ and achieve a better regret bound.

\section{Lower bounds}
We state the following lower bounds that establish the least amount of regret that any algorithm must incur. The lower bounds match the upper bounds of PolyExp proving that it is regret optimal. The proofs of the lower bounds can be found in the appendix.
\begin{restatable}{theorem}{LBFull}
    For any learner there exists an adversary producing $L_{\infty} $ losses such that the expected regret in the full information setting is:
    \begin{equation*}
        \mathbb{E}\left[\mathcal{R}_T \right]  = \Omega\left(n  \sqrt{T}  \right).
    \end{equation*}
\end{restatable}
\begin{restatable}{theorem}{LBBandit}
    For any learner there exists an adversary producing $L_{\infty} $ losses such that the expected regret in the Bandit setting is:
    \begin{equation*}
        \mathbb{E}\left[\mathcal{R}_T \right]  = \Omega\left(n^{2}  \sqrt{T}  \right).
    \end{equation*}
\end{restatable}

\section{$\{-1,+1\}^n$  Hypercube Case}
Full information and bandit algorithms which work on $\{0,1\}^n$ can be modified to work on $\{-1,+1\}^n$. The general strategy is as follows:
\begin{figure}[h]
\noindent\fbox{%
    \parbox{\linewidth}{%
\begin{enumerate}
\item Sample $X_t \in \{0,1\}^n$, play $Z_t=2X_t - \textbf{1}$ and incur loss $Z_t^\top l_t$.
\begin{enumerate}
\item Full information: $X_t \sim p_t$
\item Bandit: $X_t \sim q_t = (1-\gamma)p_t + \gamma \mu$
\end{enumerate}
\item See feedback and construct $\tilde{l}_t$
\begin{enumerate}
\item Full information: $\tilde{l}_t = l_t$
\item Bandit: $\tilde{l}_t = P_t^{-1}Z_t{Z_t}^\top l_t$ where $P_t = \mathbb{E}_{X \sim q_t}[(2X-\textbf{1})(2X-\textbf{1})^\top]$
\end{enumerate}
\item Update algorithm using $2\tilde{l}_t$
\end{enumerate}
    }%
}
\end{figure} 
\begin{restatable}{theorem}{Hypereq} Exp2 on $\{-1,+1\}^n$ using the sequence of losses $l_t$ is equivalent to PolyExp on $\{0,1\}^n$ using the sequence of losses $2\tilde{l}_t$. Moreover, the regret of Exp2 on $\{-1,1\}^n$ will equal the regret of PolyExp using the losses $2\tilde{l}_t$.
\end{restatable}

Hence, using the above strategy, PolyExp can be run in polynomial time on $\{-1,1\}^n$ and since the losses are doubled its regret only changes by a constant factor.

\section{Conclusions}
For linear losses, we show that the Exp2 algorithm can be run on the hypercube in polynomial time using PolyExp. We also show equivalences to OMD, FTRL and FTPL. We improve Exp2's regret bound in full information using OMD's analysis, showing that it is minimax optimal. In the bandit setting, the regret bound could not be improved using this analysis. It remains to show how to achieve the minimax optimal regret in the bandit setting under $L_\infty$ losses.
\section{Proofs}
\subsection{Equivalence to Exp2}
\begin{lemma} \label{Lemma2} For any sequence of losses $\tilde{l}_t$, the following is true for all $t=1,2,..,T$:
$$\prod_{i=1}^n(1+\exp(-\eta \sum_{\tau=1}^{t-1} \tilde{l}_{i,\tau})) = \sum_{Y \in \{0,1\}^n}\exp(-\eta\sum_{\tau=1}^{t-1} Y^\top \tilde{l}_\tau)$$
\end{lemma}
\begin{proof}
Consider $\prod_{i=1}^n(1+\exp(-\eta \sum_{\tau=1}^{t-1} \tilde{l}_{i,\tau}))$. It is a product of $n$ terms, each consisting of $2$ terms, $1$ and $\exp(-\eta \sum_{\tau=1}^{t-1} \tilde{l}_{i,\tau})$. On expanding the product, we get a sum of $2^n$ terms. Each of these terms is a product of $n$ terms, either a $1$ or $\exp(-\eta \sum_{\tau=1}^{t-1} \tilde{l}_{i,\tau})$. If it is $1$, then $Y_i=0$ and if it is $\exp(-\eta \sum_{\tau=1}^{t-1} \tilde{l}_{i,\tau})$, then $Y_i=1$. So,
\begin{align*}
\prod_{i=1}^n(1+\exp(-\eta \sum_{\tau=1}^{t-1} \tilde{l}_{i,\tau}))  &= \sum_{Y \in \{0,1\}^n} \prod_{i=1}^n \exp(-\eta \sum_{\tau=1}^{t-1} \tilde{l}_{i,\tau})^{Y_i}\\
&= \sum_{Y \in \{0,1\}^n} \prod_{i=1}^n \exp(-\eta \sum_{\tau=1}^{t-1} \tilde{l}_{i,\tau}{Y_i})\\
&= \sum_{Y \in \{0,1\}^n}\exp(-\eta\sum_{\tau=1}^{t-1} Y^\top \tilde{l}_\tau)
\end{align*}
\end{proof}
\Equiv*
\begin{proof}
The proof is via straightforward substitution of the expression for $x_{i,t}$ and applying Lemma \ref{Lemma2}.
\begin{align*}
\prod_{i=1}^n (x_{i,t})^{X_i} (1-x_{i,t})^{(1-X_i)} &= \prod_{i=1}^n  \frac{\left(\exp(\eta \sum_{\tau=1}^{t-1} \tilde{l}_{i,\tau}) \right)^{1-X_i}}{1+\exp(\eta \sum_{\tau=1}^{t-1} \tilde{l}_{i,\tau})}\\
&= \prod_{i=1}^n \frac{\exp(-\eta \sum_{\tau=1}^{t-1} \tilde{l}_{i,\tau})^{X_i}}{1+\exp(-\eta \sum_{\tau=1}^{t-1} \tilde{l}_{i,\tau})}\\
&= \frac{\prod_{i=1}^n\exp(-\eta \sum_{\tau=1}^{t-1} \tilde{l}_{i,\tau}X_i)}{\prod_{i=1}^n(1+\exp(-\eta \sum_{\tau=1}^{t-1} \tilde{l}_{i,\tau}))}\\
&= \frac{\exp(-\eta\sum_{\tau=1}^{t-1} X^\top \tilde{l}_{i,\tau})}{\sum_{Y \in \{0,1\}^n} \exp(-\eta \sum_{\tau=1}^{t-1} Y^\top \tilde{l}_\tau)} \end{align*}
\end{proof}
\subsection{Equivalence to OMD}
\begin{lemma} \label{lemma4} The Fenchel Conjugate of $R(x) = \sum_{i=1}^n x_i \log x_i + (1-x_i) \log (1-x_i)$ is:
$$R^\star(\theta) = \sum_{i=1}^n \log(1+\exp(\theta_i))$$
\end{lemma}
\begin{proof}Differentiating $x^\top \theta - R(x)$ wrt $x_i$ and equating to $0$:
\begin{align*}
\theta_i - \log x_i + \log (1-x_i) &= 0\\
\frac{x_i}{1-x_i} = \exp(\theta_i)\\
x_i = \frac{1}{1+\exp(-\theta_i)}
\end{align*}
Substituting this back in $x^\top \theta - R(x)$, we get $R^\star(\theta) = \sum_{i=1}^n \log(1+\exp(\theta_i))$. It is also straightforward to see that $\nabla R^\star(\theta)_i = (1+\exp(-\theta_i))^{-1}$
\end{proof}
\EqOMD*
\begin{proof}
It is easy to see that $E[X_{i,t}] = \Pr(X_{i,t}=1) = x_{i,t}$. Hence $E[X_t] = x_t$.

The update equation is $y_{t+1} = \nabla R^\star(\nabla R(x_t) - \eta \tilde{l}_t)$. Evaluating $\nabla F$ and using $\nabla R^\star$ from Lemma \ref{lemma4}:
\begin{align*}
y_{t+1,i} &= \frac{1}{1 + \exp(- \log(x_{t,i}) + \log (1-x_{t,i}) + \eta \tilde{l}_{t,i})}\\
&= \frac{1}{1 + \frac{1-x_{t,i}}{x_{t,i}}  \exp(\eta \tilde{l}_{t,i})}\\
&= \frac{x_{t,i}}{x_{t,i} + (1-x_{t,i})\exp(\eta\tilde{l}_{t,i})}
\end{align*}
Since $0\leq(1+\exp(-\theta))^{-1}\leq 1$, we have that $y_{i,t+1}$ is always in $[0,1]$. Bregman projection step is not required. So we have $x_{i,t+1} = y_{i,t+1}$ which gives the same update as PolyExp.
\end{proof}
\subsection{Equivalence to FTRL and FTPL}
\EqFTPRL*
\begin{proof}
Using $R(x) = \sum_{i=1}^n x_i\log x_i +(1-x_i) \log (1-x_i)$ in FTRL, we can solve the optimization problem exactly:
\begin{align*}
\nabla_i \left(\eta\sum_{\tau=1}^t l_\tau^\top x + R(x)\right) &= \eta\sum_{\tau=1}^t l_{i,\tau} + \log \frac{x_i}{ 1-x_i} = 0\\
\implies x_i&=  \frac{1}{1+\exp(\eta \sum_{\tau=1}^t l_{i,\tau})}
\end{align*}
This is the same update equation as PolyExp.

For FTPL, the optimization problem outputs $X_i=1$ if $\eta \sum_{\tau=1}^t l_{i,\tau} + v_i \leq 0$ and $X_i=0$ otherwise. As $v$ is a vector of iid random variables, we can conclude that $X_{i,t+1}$ are also independent. We have that:
\begin{align*}
\Pr(X_{i,t+1}=1) &= \Pr(\eta \sum_{\tau=1}^t l_{i,\tau} + v_i \leq 0)\\
&=\Pr( v_i \leq -\eta \sum_{\tau=1}^t l_{i,\tau} )\\
&=\frac{1}{1+\exp(\eta \sum_{\tau=1}^t l_{i,\tau})}
\end{align*}
To obtain the last equality, we use the fact that $v_i$ is drawn from the Logistic distribution. We can see this is equal to the probability of picking $X_{i,t+1}=1$ in the PolyExp algorithm.
\end{proof}
\subsection{PolyExp Full Information Regret Proof}
\begin{lemma}[see Theorem 5.5 in \citep{bubeck2012regret}] \label{lemma3} For any $x \in \mathcal{\bar{X}}$, OMD with Legendre regularizer $R(x)$ with domain $\mathcal{\bar{X}}$ and $R^\star$ is differentiable on $\mathbb{R}^n$ satisfies:
\begin{align*}
\sum_{t=1}^T x_t^\top l_t - \sum_{t=1}^T x^\top l_t &\leq \frac{R(x)-R(x_1)}{\eta} + \frac{1}{\eta} \sum_{t=1}^T D_{R^\star} (\nabla R(x_t) - \eta l_t \|\nabla R(x_t))
\end{align*}
\end{lemma}
\begin{lemma}\label{lemma5} If $|\eta l_{t,i}| \leq 1$ for all $t \in [T]$ and $i \in [n]$, OMD with entropic regularizer $R(x) = \sum_{i=1}^n x_i \log x_i + (1-x_i) \log (1-x_i) $ satisfies for any $x \in [0,1]^n$,:
$$\sum_{t=1}^T x_t^\top l_t - \sum_{t=1}^T x^\top l_t \leq \frac{n \log 2}{\eta} + \eta \sum_{t=1}^T x_t^T l_t^2$$
\end{lemma}
\begin{proof} We start from Lemma \ref{lemma3}. Using the fact that $x \log(x) + (1-x) \log(1-x) \geq -\log 2$, we get $R(x)-R(x_1) \leq n \log 2$. Next we bound the Bregmen term using Lemma \ref{lemma4}
\begin{align*}
D_{R^\star} (\nabla R(x_t) &- \eta l_t \|\nabla R(x_t)) = R^\star(\nabla R(x_t) - \eta l_t) \\&- R^\star(\nabla R(x_t)) + \eta l_t^\top \nabla R^\star(\nabla R(x_t))
\end{align*}
Using that fact that $\nabla R^\star = (\nabla R)^{-1}$, the last term is $\eta x_t^\top l_t$.
The first two terms can be simplified as:
\begin{align*}
&R^\star(\nabla R(x_t) - \eta l_t) - R^\star(\nabla R(x_t)) \\&= \sum_{i=1}^n \log \frac{1+\exp(\nabla R(x_t)_i - \eta l_{t,i})}{1+\exp(\nabla R(x_t)_i)} \\
&=\sum_{i=1}^n \log \frac{1+\exp(-\nabla R(x_t)_i + \eta l_{t,i})}{\exp(\eta l_{t,i})(1+\exp(-\nabla R(x_t)_i )}
\end{align*}
Using the fact that $\nabla R(x_t)_i = \log x_i - \log (1-x_i)$:
\begin{align*}
&= \sum_{i=1}^n \log \frac{x_{t,i}+(1-x_{t,i})\exp(\eta l_{t,i})}{\exp(\eta l_{t,i})}\\
&=\sum_{i=1}^n \log(1-x_{t,i}+x_{t,i}\exp(-\eta l_{t,i}))
\end{align*}
Using the inequality: $e^{-x} \leq 1-x+x^2$ when $x\geq -1$. So when $|\eta l_{t,i}| \leq 1$:
$$\leq \sum_{i=1}^n \log(1-\eta x_{t,i}l_{t,i}+\eta^2 x_{t,i}l_{t,i}^2)$$
Using the inequality: $\log(1-x) \leq -x$
$$\leq -\eta x_t ^\top l_t +\eta^2  x_t^\top l_t^2 $$
The Bregman term can be bounded by $-\eta x_t ^\top l_t +\eta^2  x_t^\top l_t^2  + \eta x_t ^\top l_t = \eta^2  x_t^\top l_t^2 $
Hence, we have:
$$\sum_{t=1}^T x_t^\top l_t - \sum_{t=1}^T x^\top l_t \leq \frac{n \log 2}{\eta} + \eta \sum_{t=1}^T x_t^T l_t^2$$
\end{proof}
\RegPoly*
\begin{proof}
Applying expectation with respect to the randomness of the player to definition of regret, we get:
\begin{align*}
\mathbb{E}[\mathcal{R}_T] &= \mathbb{E}[\sum_{t=1}^T X_t^\top l_t - \min_{X^\star \in \{0,1\}^n} \sum_{t=1}^T {X^\star} ^\top l_t] \\&= \sum_{t=1}^T x_t^\top l_t - \min_{X^\star \in \{0,1\}^n} \sum_{t=1}^T {X^\star} ^\top l_t
\end{align*}
Applying Lemma \ref{lemma5}, we get $E[\mathcal{R}_T] \leq \frac{n \log 2}{\eta} + \eta \sum_{t=1}^T x_t^T l_t^2$. Using the fact that $|l_{i,t}|\leq 1$, we get $\sum_{t=1}^T x_t^T l_t^2 \leq nT$. 
$$\mathbb{E}[\mathcal{R}_T] \leq \eta nT + \frac{n \log 2}{\eta}$$
Optimizing over the choice of $\eta$, we get that the regret is bounded by $2 n \sqrt{T \log 2}$ if we choose $\eta = \sqrt{\frac{\log 2}{T}}$.
\end{proof}
\subsubsection{Bandit}
\begin{lemma} \label{lemma5_2} Let $ \tilde{l}_{t} = P_t^{-1}X_tX_t^\top l_t$ and $x'_t = (1-\gamma)x_t + \gamma \mu$. If $|\eta \tilde{l}_{t,i}| \leq 1$ for all $t\in [T]$ and $i \in [n]$, OMD with entropic regularization and uniform exploration satisfies for any $x \in [0,1]^n$:
$$\sum_{t=1}^T {x'}_t^\top l_t - \sum_{t=1}^T x^\top l_t \leq \eta \mathbb{E}[\sum_{t=1}^T {x'}_t^\top \tilde{l}_t^2] + \frac{n\log 2}{\eta} + 2\gamma nT$$
\end{lemma}
\begin{proof}
We have that:
\begin{align*}
\sum_{t=1}^T {x'}_t^\top \tilde{l}_t - \sum_{t=1}^T x^\top \tilde{l}_t &= (1-\gamma)(\sum_{t=1}^T x_{t}^\top \tilde{l}_t - \sum_{t=1}^T x^\top \tilde{l}_t) \\&+ \gamma (\sum_{t=1}^T x_{\mu}^\top \tilde{l}_t - \sum_{t=1}^T x^\top \tilde{l}_t)
\end{align*}
Since the algorithm runs OMD on $\tilde{l}_t$ and $|\eta \tilde{l}_t| \leq 1$, we can apply Lemma \ref{lemma5}:
\begin{align*}
\sum_{t=1}^T {x'}_t^\top \tilde{l}_t - \sum_{t=1}^T x^\top \tilde{l}_t &\leq (1-\gamma)(\eta \sum_{t=1}^T x_{t}^T \tilde{l}_t^2 + \frac{n \log 2}{\eta}) \\&+ \gamma(\sum_{t=1}^T x_{\mu}^\top \tilde{l}_t - \sum_{t=1}^T x^\top \tilde{l}_t)
\end{align*}
Apply expectation with respect to $X_t$.  Using the fact that $\mathbb{E}[\tilde{l}_t] = l_t$ and $x_{\mu}^\top l_t - x^\top l_t \leq 2n$:
\begin{align*}
\sum_{t=1}^T {x'}_t^\top l_t - \sum_{t=1}^T x^\top l_t &\leq (1-\gamma)(\eta \mathbb{E}[\sum_{t=1}^T x_{p_t}^T \tilde{l}_t^2] + \frac{n \log 2}{\eta}) \\&+ 2\gamma nT\\
&\leq \eta \mathbb{E}[\sum_{t=1}^T {x'}_{t}^T \tilde{l}_t^2] + \frac{n \log 2}{\eta} + 2\gamma nT
\end{align*}
\end{proof}

Now, we attempt analyze the regret of PolyExp in the bandit setting using $\tilde{l}_t = P_t^{-1}X_tX_t^\top l_t$. Applying expectation with respect to the randomness of the player to the definition of regret, we get:
\begin{align*}
\mathbb{E}[\mathcal{R}_T] &= \mathbb{E}[\sum_{t=1}^T X_t^\top l_t - \min_{X^\star \in \{0,1\}^n}\sum_{t=1}^T {X^\star}^\top l_t ] \\&= \sum_{t=1}^T {x'}_t^\top l_t - \min_{X^\star \in \{0,1\}^n}\sum_{t=1}^T {X^\star}^\top l_t
\end{align*}
Assuming $|\eta \tilde{l}_{t,i}| \leq 1$, we apply Lemma \ref{lemma5_2}
$$\mathbb{E}[\mathcal{R}_T] \leq  \eta \mathbb{E}[\sum_{t=1}^T {x'}_{t}^T \tilde{l}_t^2] + \frac{n \log 2}{\eta} + 2\gamma nT$$
To satisfy $|\eta \tilde{l}_{t,i}| \leq 1$, we need the following condition:
\begin{align*}
|\eta \tilde{l}_{t,i}| &= \eta |\tilde{l}_t^\top e_i| = \eta |(P_t^{-1}X_tX_t^\top l_t)^\top e_i| \\&\leq n \eta |X_t^\top P_t^{-1}e_i| \leq n \eta |X_t^\top e_i|\|P_t^{-1}\| 
\end{align*}
Since $P_t \succeq \frac{\gamma}{4}I_n$ and $|X_t^\top e_i|\leq 1$, we should have $\frac{4n \eta}{\gamma} \leq 1$. Taking $\gamma = 4n \eta$, we get:
$$\mathbb{E}[\mathcal{R}_T] \leq \eta \mathbb{E}[\sum_{t=1}^T {x'}_{t}^T \tilde{l}_t^2] + \frac{n\log 2}{\eta} + 8\eta n^2 T$$
Analyzing the first term:
\begin{align*}
{x'}_{t}^T \tilde{l}_t^2 &= l_t^\top X_t X_t^\top P_t^{-1} \text{diag}({x'}_t)P_t^{-1}X_t X_t^\top l_t \leq n^2 \text{Trace}(P_t^{-1}\text{diag}({x'}_t)P_t^{-1}X_tX_t^\top)\\
\mathbb{E}[{x'}_{t}^T \tilde{l}_t^2] &\leq n^2 \text{Trace}(P_t^{-1}\text{diag}({x'}_t)) = n^2\text{Trace}(P_t^{-1}\circ P_t)
\end{align*}
So this term could become unbounded because of the $\text{Trace}(P_t^{-1}\circ P_t)$. Hence, PolyExp's regret equation does not yield a regret upperbound when using $\tilde{l}_t = P_t^{-1}X_tX_t^\top l_t$. However, it should be possible to find a linear estimator tailor made for the hypercube in order to analyze PolyExp in the bandit setting.
\appendix
\section{Supplementary Proofs}
\subsection{Exp2 Regret Proofs}
First, we directly analyze Exp2's regret for the two kinds of feedback.
\subsubsection{Full Information}
\begin{lemma}\label{Lemma1}Let $L_t(X) = X^\top l_t$. If $|\eta L_t(X)| \leq 1$ for all $t \in [T]$ and $X \in \{0,1\}^n$, the Exp2 algorithm satisfies for any $X$:
$$\sum_{t=1}^T p_t^\top L_t - \sum_{t=1}^T L_t(X) \leq \eta \sum_{t=1}^T p_t^\top L_t^2 + \frac{n\log 2}{\eta}$$
\end{lemma}
\begin{proof}
(Adapted from \citep{hazan2016introduction} Theorem 1.5) Let $Z_{t} = \sum_{Y \in \{0,1\}^n}w_t(Y)$. We have:
\begin{align*}
Z_{t+1} &= \sum_{Y \in \{0,1\}^n}\exp(-\eta L_t(Y))w_t(Y)\\
&= Z_t \sum_{Y \in \{0,1\}^n}\exp(-\eta L_t(Y)) p_t(Y)
\end{align*}
Since $e^{-x} \leq 1-x+x^2$ for $x\geq -1$, we have that $\exp(-\eta L_t(Y)) \leq 1-\eta L_t(Y)+\eta^2 L_t(Y)^2$ (Because we assume $|\eta L_t(X)| \leq 1$). So,
\begin{align*}
Z_{t+1} &\leq Z_t \sum_{Y \in \{0,1\}^n}(1-\eta L_t(Y)+\eta^2 L_t(Y)^2) p_t(Y) \\
&= Z_t (1 - \eta p_t^\top L_t + \eta^2 p_t^\top L_t^2)
\end{align*}
Using the inequality $1+x \leq e^x$,
\begin{align*}
Z_{t+1} &\leq Z_t \exp( -\eta p_t^\top L_t + \eta^2 p_t^\top L_t^2)
\end{align*}
Hence, we have:
$$Z_{T+1} \leq Z_1 \exp( -\sum_{t=1}^T\eta p_t^\top L_t + \sum_{t=1}^T\eta^2 p_t^\top L_t^2)$$
For any $X \in \{0,1\}^n$, $w_{T+1}(X) = \exp(-\sum_{t=1}^T\eta L_t(X) )$. Since $w(T+1)(X) \leq Z_{T+1}$ and $Z_1 = 2^n$, we have:
\begin{align*}
\exp(-\sum_{t=1}^T\eta L_t(X) ) \leq 2^n \exp( -\sum_{t=1}^T\eta p_t^\top L_t + \sum_{t=1}^T\eta^2 p_t^\top L_t^2)
\end{align*}
Taking the logarithm on both sides manipulating this inequality, we get:
$$\sum_{t=1}^T p_t^\top L_t - \sum_{t=1}^T L_t(X) \leq \eta \sum_{t=1}^T p_t^\top L_t^2 + \frac{n\log 2}{\eta}$$
\end{proof}
\ExpFReg*
\begin{proof}
Using $L_t(X) = X^\top l_t$ and applying expectation with respect to the randomness of the player to definition of regret, we get:
\begin{align*}
E[\mathcal{R}_T] &= \sum_{t=1}^T \sum_{X \in \{0,1\}^n}p_t(X)L_t(X) - \min_{X^\star \in \{0,1\}^n}\sum_{t=1}^T L_t(X^\star) \\
&=  \sum_{t=1}^T p_t^\top L_t - \min_{X^\star \in \{0,1\}^n}\sum_{t=1}^T L_t(X^\star)
\end{align*}
Applying Lemma \ref{Lemma1}, we get $E[\mathcal{R}_T] \leq \eta \sum_{t=1}^T p_t^\top L_t^2 + n\log 2/\eta$. Since $|L_t(X)| \leq n $ for all $X  \in \{0,1\}^n$, we get $\sum_{t=1}^T p_t^\top L_t^2 \leq  T n^2 $.
$$E[\mathcal{R}_T] \leq \eta Tn^2 + \frac{n\log 2}{\eta}$$
Optimizing over the choice of $\eta$, we get the regret is bounded by $2n^{3/2} \sqrt{T \log 2}$ if we choose $\eta = \sqrt{\frac{\log 2}{nT}}$.

To apply Lemma \ref{Lemma1}, $|\eta L_t(X)| \leq 1$ for all $t \in [T]$ and $X \in \{0,1\}^n$. Since $|L_t(X)| \leq n$, we have $\eta \leq 1/n$.
\end{proof}
\subsubsection{Bandit}
\begin{lemma}\label{lemma5_1} Let $\tilde{L_t}(X) = X^\top \tilde{l}_t$, where $\tilde{l}_t = P_t^{-1}X_tX_t^\top l_t$. If $|\eta\tilde{L_t}(X)| \leq 1$ for all $t \in [T]$ and $X \in \{0,1\}^n$, the Exp2 algorithm with uniform exploration satisfies for any $X$
$$\sum_{t=1}^T  q_t^\top L_t - \sum_{t=1}^T L_t(X) \leq \eta \mathbb{E}[\sum_{t=1}^T q_t^\top \tilde{L}_t^2] + \frac{n\log 2}{\eta} + 2\gamma n T$$
\end{lemma}
\begin{proof}
We have that:
\begin{align*}
\sum_{t=1}^T q_t^\top \tilde{L}_t - \sum_{t=1}^T \tilde{L}_t(X)  &= (1-\gamma) (\sum_{t=1}^T p_t^\top \tilde{L}_t - \sum_{t=1}^T \tilde{L}_t(X)) + \gamma (\sum_{t=1}^T \mu^\top \tilde{L}_t - \sum_{t=1}^T \tilde{L}_t(X))
\end{align*}
Since the algorithm essentially runs Exp2 using the losses $\tilde{L}_t(X)$ and $|\eta\tilde{L_t}(X)| \leq 1$, we can apply Lemma \ref{Lemma1}:
\begin{align*}
\sum_{t=1}^T q_t^\top \tilde{L}_t - \sum_{t=1}^T \tilde{L}_t(X)  &\leq (1-\gamma) (\frac{n \log 2}{\eta} + \eta \sum_{t=1}^T p_t^\top \tilde{L}_t^2) + \gamma (\sum_{t=1}^T \mu^\top \tilde{L}_t - \sum_{t=1}^T \tilde{L}_t(X))
\end{align*}
Apply expectation with respect to $X_t$.  Using the fact that $\mathbb{E}[\tilde{l}_t] = l_t$ and $\mu^\top L_t -  L_t(X) \leq 2n$:
\begin{align*}
\sum_{t=1}^T q_t^\top L_t - \sum_{t=1}^T L_t(X)  &\leq (1-\gamma) (\frac{n \log 2}{\eta} + \eta \mathbb{E}[\sum_{t=1}^T p_t^\top \tilde{L}_t^2]) + \gamma (\sum_{t=1}^T \mu^\top L_t - \sum_{t=1}^T L_t(X))\\
&\leq \eta \mathbb{E}[\sum_{t=1}^T q_t^\top \tilde{L}_t^2] + \frac{n \log 2}{\eta}  + 2 \gamma nT
\end{align*}
\end{proof}
\ExpBanReg*
\begin{proof}
Applying expectation with respect to the randomness of the player to the definition of regret, we get:
\begin{align*}
\mathbb{E}[\mathcal{R}_T] &= \mathbb{E}[\sum_{t=1}^T L_t(X_t) - \min_{X^\star \in \{0,1\}^n} L_t(X^\star)] \\&= \sum_{t=1}^T  q_t^\top L_t - \min_{X^\star \in \{0,1\}^n}\sum_{t=1}^T L_t(X^\star)
\end{align*}
Applying Lemma \ref{lemma5_1}
$$\mathbb{E}[\mathcal{R}_T] \leq \eta \mathbb{E}[\sum_{t=1}^T q_t^\top \tilde{L}_t^2] + \frac{n \log 2}{\eta}  + 2 \gamma nT$$
We follow the proof technique of \citep{bubeck2012towards} Theorem 4. We have that:
\begin{align*}
q_t^\top \tilde{L}_t^2 &= \sum_{X \in \{0,1\}^n}q_t(X)(X^\top \tilde{l_t})^2  \\&= \sum_{X \in \{0,1\}^n}q_t(X)(\tilde{l_t}^\top X X^\top \tilde{l_t})\\
&= \tilde{l_t}^\top P_t \tilde{l_t} \\&= l_t^\top X_tX_t^\top P_t^{-1}  P_t P_t^{-1}X_t X_t^\top l_t \\&= (X_t^\top l_t)^2 X_t^\top P_t^{-1}X_t\\
&\leq n^2 X_t^\top P_t^{-1}X_t = n^2 \text{Tr}(P_t^{-1} X_t X_t^\top )
\end{align*}
Taking expectation, we get $E[q_t^\top \tilde{L}_t^2] \leq n^2 \text{Tr}(P_t^{-1} \mathbb{E}[X_t X_t^\top] )= n^2 \text{Tr}(P_t^{-1} P_t) = n^3$.
Hence,
$$\mathbb{E}[\mathcal{R}_T] \leq \eta n^3 T + \frac{n \log 2}{\eta}  + 2 \gamma nT$$
However, in order to apply Lemma \ref{lemma5_1}, we need that $|\eta X^\top \tilde{l}_t| \leq 1$. We have that
$$|\eta X^\top \tilde{l}_t| = \eta |(X_t ^\top l_t) X^\top P_t^{-1} X_t| \leq 1$$
As $|X_t ^\top l_t| \leq n$ and $|X_t^\top X|\leq n$, we get $\eta n |X^\top P_t^{-1} X_t|  \leq \eta n |X^\top X_t|\|P_t^{-1}\|\leq \eta n^2\|P_t^{-1}\| \leq 1 $. The matrix $P_t = (1-\gamma)\Sigma_t + \gamma \Sigma_\mu$. The smallest eigenvalue of $\Sigma_\mu$ is $1/4$\citep{cesa2012combinatorial}. So $P_t \succeq \frac{\gamma}{4} I_n $ and $P_t^{-1} \preceq \frac{4}{\gamma}I_n$. We should have that $\frac{4n^2 \eta}{\gamma}\leq 1$. Substituting $\gamma = 4n^2 \eta$ in the regret inequality, we get:
\begin{align*}
\mathbb{E}[\mathcal{R}_T] &\leq \eta n^3 T + 8\eta n^3 T +  \frac{n \log 2}{\eta}\\
&\leq 9\eta n^3T +  \frac{n \log 2}{\eta}
\end{align*}
Optimizing over the choice of $\eta$, we get $\mathbb{E}[\mathcal{R}_T] \leq 2 n^{2} \sqrt{9T \log 2}$ when $\eta = \sqrt{\frac{\log 2}{9n^2T}}$.
\end{proof}

\subsection{Lower Bounds}
\subsubsection{Full Information Lower bound}

In the game between player and adversary, the players strategy is to pick some probability distribution $p_t \in \Delta(\{0,1\}^n)$ for $t=1\dots T$. The adversary picks a density $q_t$ over loss vectors $l \in [-1,1]^n$ for $t=1 \dots T$. So player picks $X_t \sim p_t$ and adversary picks $l_t \sim q_t$.  The min max expected regret is:
$$\inf_{p_1 \dots p_T} \sup_{q_1 \dots q_t}   \mathbb{E}_{l_t \sim q_t} \mathbb{E}_{X_t \sim p_t}\left[\sum_{t=1}^T l_t^\top X_t - \min_X \sum_{t=1}^T l_t^\top X  \right] $$
Let $\mathbb{E}_{X_t \sim p_t} = x_t$. 
$$\inf_{p_1 \dots p_T} \sup_{q_1 \dots q_t} \mathbb{E}_{l_t \sim q_t} [\sum_{t=1}^T l_t^\top x_t - \min_X \sum_{t=1}^T l_t^\top X] $$

\LBFull*

\begin{proof}
We choose $q_t$ to be the density such that $l_{t,i}$ is a Rademacher random variable, ie, $l_{t,i} = +1$ w.p. $1/2$ and $l_{t,i}=-1$ w.p $1/2$ for all $t=1\dots T$ and $i=[n]$. So,
\begin{align*}
    &\inf_{p_1 \dots p_T} \sup_{q_1 \dots q_t} \mathbb{E}_{l_t \sim q_t} \left[\sum_{t=1}^T l_t^\top x_t - \min_X \sum_{t=1}^T l_t^\top X   \right] \\&\geq \inf_{p_1 \dots p_T} \mathbb{E}_{l_{t}} \left[\sum_{t=1}^T l_t^\top x_t - \min_X \sum_{t=1}^T l_t^\top X  \right]
\end{align*}
For our choice of $q_t$, we have $\mathbb{E}_{l_{t}} [l_t^\top x_t] = 0$. So,
\begin{align*}
    &\inf_{p_1 \dots p_T} \mathbb{E}_{l_{t}} [\sum_{t=1}^T l_t^\top x_t - \min_X \sum_{t=1}^T l_t^\top X] \\&= \inf_{p_1 \dots p_T} \mathbb{E}_{l_{t}} [- \min_X \sum_{t=1}^T l_t^\top X] \\
    &= \mathbb{E}_{l_{t}} [\max_X \sum_{t=1}^T l_t^\top X]
\end{align*}
Simplifying this, we get:
\begin{align*}
        \mathbb{E}_{l_{t}} [\max_X \sum_{t=1}^T l_t^\top X] &= \mathbb{E}_{l_{t}} [\max_{X_1\dots X_n} \sum_{t=1}^T \sum_{i=1}^n l_{t,i} X_i]\\
        &= \mathbb{E}_{l_{t}}[\sum_{i=1}^n\max_{X_i} \sum_{t=1}^T  l_{t,i} X_i]\\
        &= \sum_{i=1}^n \mathbb{E}_{l_{t,i}}[\max_{X_i} \sum_{t=1}^T  l_{t,i} X_i]\\
        &= n \mathbb{E}_{Y}[\max_{x} \sum_{t=1}^T  Y_{t} x]
\end{align*}
Here $Y$ is a Rademacher random vector of length $T$ and $x \in \{0,1\}$. We have that
$$\max_{x} \left[\sum_{t=1}^T  Y_{t} x\right] =\begin{cases} 0 \quad \text{ If } \sum_{t=1}^T  Y_{t} \leq 0 \\ \sum_{t=1}^T  Y_{t} \quad \text{ otherwise}
\end{cases}$$
So
\begin{align*}
    \mathbb{E}_{Y}\left[\max_{x} \sum_{t=1}^T  Y_{t} x  \right] &= \mathbb{E}_{Y}[\sum_{t=1}^T  Y_{t} | \sum_{t=1}^T  Y_{t} >0]\\
    &= \frac{1}{2}\mathbb{E}_{Y}  \left| {\sum_{t=1}^T  Y_{t}} \right|  
\end{align*}
Using Khintchine's inequality, we have positive constants $A$ and $B$ such that:
$$A\left( \sum _ { t = 1 } ^ { T } \left|1 \right| ^ { 2 } \right) ^ { 1 / 2 } \leq \mathbb{E}_{Y}  \left| {\sum_{t=1}^T  Y_{t}} \right| \leq B\left( \sum _ { t = 1 } ^ { T } \left|1 \right| ^ { 2 } \right) ^ { 1 / 2 }$$
Hence, the regret is lower bounded by $\Omega(n \sqrt{T})$.
\end{proof}
\subsubsection{Bandit Lower bound}
Consider the multi-task bandit problem where the player plays $n$ simultaneous $k$ armed bandits. Each round, the player picks an arm for each of the bandits and sees only the sum of losses. The lower bound on the regret for this setting is given by the following lemma. 
\begin{lemma}[see Theorem 1 in \citep{cohen2017tight}] \label{lemma_last} For the multi-task bandit problem with $n$ simultaneous $k$ armed bandits, there exists a sequence of losses such that the expected regret is:
$$\mathbb{E}[R_t] = \Omega(n^2\sqrt{kT})$$
\end{lemma}
\LBBandit*
\begin{proof}
Playing on the hypercube can be considered as playing $n$ simultaneous $2$ armed bandits. Applying \ref{lemma_last} we get the lower bound as $\Omega(n^2\sqrt{T})$
\end{proof}

\subsection{$\{-1,+1\}^n$ Hypercube Case}
\begin{lemma} \label{lemma6}Exp2 on $\{-1,+1\}^n$ with losses $l_t$ is equivalent to Exp2 on $\{0,1\}^n$ with losses $2l_t$ while using the map $2X_t-\textbf{1}$ to play on $\{-1,+1\}^n$.
\end{lemma}
\begin{proof}
Consider the update equation for Exp2 on $\{-1,+1\}^n$
$$p_{t+1}(Z) = \frac{\exp(-\eta \sum_{\tau=1}^t Z^\top l_\tau)}{\sum_{W \in \{-1,+1\}^n}\exp(-\eta \sum_{\tau=1}^t W^\top l_\tau)}$$
$Z \in \{-1,+1\}^n$ can be mapped to a $X \in \{0,1\}^n$ using the bijective map $X = (Z+\textbf{1})/2$. So:
\begin{align*}
p_{t+1}(Z) &= \frac{\exp(-\eta \sum_{\tau=1}^t (2X-\textbf{1})^\top l_\tau)}{\sum_{Y \in \{0,1\}^n}\exp(-\eta \sum_{\tau=1}^t (2Y-\textbf{1})^\top l_\tau)}\\
&= \frac{\exp(-\eta \sum_{\tau=1}^t X^\top (2l_\tau))}{\sum_{Y \in \{0,1\}^n}\exp(-\eta \sum_{\tau=1}^t Y^\top (2l_\tau))}
\end{align*}
This is equivalent to updating the Exp2 on $\{0,1\}^n$ with the loss vector $2l_t$.
\end{proof}
\Hypereq*
\begin{proof}
After sampling $X_t$, we play $Z_t = 2X_t-\textbf{1}$. So $\Pr(X_t=X) = \Pr(Z_t = 2X -\textbf{1})$. In full information, $2\tilde{l}_t = 2l_t$ and in the bandit case $\mathbb{E}[2\tilde{l}_t] = 2l_t$. Since $2\tilde{l}_t$ is used to update the algorithm, by Lemma \ref{lemma6} we have that $\Pr(X_{t+1}=X) = \Pr(Z_{t+1} = 2X -\textbf{1})$. By equivalence of Exp2 to PolyExp, the first statement follows immediately. Let $Z^\star = \min \limits_{Z \in \{-1,+1\}^n} \sum_{t=1}^T Z^\top l_t$ and $2X^\star = Z^\star + \textbf{1}$. The regret of Exp2 on $\{-1,+1\}^n$ is:
\begin{align*}
\sum_{t=1}^T l_t^\top (Z_t - Z^\star) &= \sum_{t=1}^T l_t^\top (2X_t -\textbf{1} - 2X^\star + \textbf{1})\\&= \sum_{t=1}^T (2l_t)^\top (X_t - X^\star)
\end{align*}
\end{proof}
\bibliography{sample}

\begin{thebibliography}{19}
\providecommand{\natexlab}[1]{#1}
\providecommand{\url}[1]{\texttt{#1}}
\expandafter\ifx\csname urlstyle\endcsname\relax
  \providecommand{\doi}[1]{doi: #1}\else
  \providecommand{\doi}{doi: \begingroup \urlstyle{rm}\Url}\fi

\bibitem[Abernethy et~al.(2009)Abernethy, Hazan, and Rakhlin]{aber2009compe}
Jacob~D Abernethy, Elad Hazan, and Alexander Rakhlin.
\newblock Competing in the dark: An efficient algorithm for bandit linear
  optimization.
\newblock 2009.

\bibitem[Audibert et~al.(2011)Audibert, Bubeck, and
  Lugosi]{audibert2011minimax}
Jean-Yves Audibert, S{\'e}bastien Bubeck, and G{\'a}bor Lugosi.
\newblock Minimax policies for combinatorial prediction games.
\newblock In \emph{Proceedings of the 24th Annual Conference on Learning
  Theory}, pages 107--132, 2011.

\bibitem[Audibert et~al.(2013)Audibert, Bubeck, and Lugosi]{audibert2013regret}
Jean-Yves Audibert, S{\'e}bastien Bubeck, and G{\'a}bor Lugosi.
\newblock Regret in online combinatorial optimization.
\newblock \emph{Mathematics of Operations Research}, 39\penalty0 (1):\penalty0
  31--45, 2013.

\bibitem[Bubeck(2011)]{bubeck2011introduction}
S{\'e}bastien Bubeck.
\newblock Introduction to online optimization.
\newblock \emph{Lecture Notes}, pages 1--86, 2011.

\bibitem[Bubeck and Cesa-Bianchi(2012)]{bubeck2012regret}
S{\'e}bastien Bubeck and Nicolo Cesa-Bianchi.
\newblock Regret analysis of stochastic and nonstochastic multi-armed bandit
  problems.
\newblock \emph{Foundations and Trends{\textregistered} in Machine Learning},
  5\penalty0 (1):\penalty0 1--122, 2012.

\bibitem[Bubeck et~al.(2012)Bubeck, Cesa-Bianchi, and
  Kakade]{bubeck2012towards}
S{\'e}bastien Bubeck, Nicolo Cesa-Bianchi, and Sham Kakade.
\newblock Towards minimax policies for online linear optimization with bandit
  feedback.
\newblock In \emph{Annual Conference on Learning Theory}, volume~23, pages
  41--1. Microtome, 2012.

\bibitem[Cesa-Bianchi and Lugosi(2006)]{cesa2006prediction}
Nicolo Cesa-Bianchi and G{\'a}bor Lugosi.
\newblock \emph{Prediction, learning, and games}.
\newblock Cambridge university press, 2006.

\bibitem[Cesa-Bianchi and Lugosi(2012)]{cesa2012combinatorial}
Nicolo Cesa-Bianchi and G{\'a}bor Lugosi.
\newblock Combinatorial bandits.
\newblock \emph{Journal of Computer and System Sciences}, 78\penalty0
  (5):\penalty0 1404--1422, 2012.

\bibitem[Cohen et~al.(2017)Cohen, Hazan, and Koren]{cohen2017tight}
Alon Cohen, Tamir Hazan, and Tomer Koren.
\newblock Tight bounds for bandit combinatorial optimization.
\newblock \emph{arXiv preprint arXiv:1702.07539}, 2017.

\bibitem[Dani et~al.(2008)Dani, Kakade, and Hayes]{dani2008price}
Varsha Dani, Sham~M Kakade, and Thomas~P Hayes.
\newblock The price of bandit information for online optimization.
\newblock In \emph{Advances in Neural Information Processing Systems}, pages
  345--352, 2008.

\bibitem[Freund and Schapire(1997)]{freund1997decision}
Yoav Freund and Robert~E Schapire.
\newblock A decision-theoretic generalization of on-line learning and an
  application to boosting.
\newblock \emph{Journal of computer and system sciences}, 55\penalty0
  (1):\penalty0 119--139, 1997.

\bibitem[Hazan(2016)]{hazan2016introduction}
Elad Hazan.
\newblock Introduction to online convex optimization.
\newblock \emph{Foundations and Trends{\textregistered} in Optimization},
  2\penalty0 (3-4):\penalty0 157--325, 2016.

\bibitem[Kalai and Vempala(2005)]{kalai2005efficient}
Adam Kalai and Santosh Vempala.
\newblock Efficient algorithms for online decision problems.
\newblock \emph{Journal of Computer and System Sciences}, 71\penalty0
  (3):\penalty0 291--307, 2005.

\bibitem[Koolen et~al.(2010)Koolen, Warmuth, and Kivinen]{koolen2010hedging}
Wouter~M Koolen, Manfred~K Warmuth, and Jyrki Kivinen.
\newblock Hedging structured concepts.
\newblock In \emph{COLT}, pages 93--105. Citeseer, 2010.

\bibitem[Littlestone and Warmuth(1994)]{littlestone1994weighted}
Nick Littlestone and Manfred~K Warmuth.
\newblock The weighted majority algorithm.
\newblock \emph{Information and computation}, 108\penalty0 (2):\penalty0
  212--261, 1994.

\bibitem[Nemirovsky and Yudin(1983)]{nemirovsky1983problem}
Arkadii~Semenovich Nemirovsky and David~Borisovich Yudin.
\newblock Problem complexity and method efficiency in optimization.
\newblock 1983.

\bibitem[Rakhlin and Tewari(2009)]{rakhlin2009lecture}
Alexander Rakhlin and A~Tewari.
\newblock Lecture notes on online learning.
\newblock \emph{Draft, April}, 2009.

\bibitem[Shalev-Shwartz(2012)]{shalev2012online}
Shai Shalev-Shwartz.
\newblock Online learning and online convex optimization.
\newblock \emph{Foundations and Trends{\textregistered} in Machine Learning},
  4\penalty0 (2):\penalty0 107--194, 2012.

\bibitem[Srebro et~al.(2011)Srebro, Sridharan, and
  Tewari]{srebro2011universality}
Nati Srebro, Karthik Sridharan, and Ambuj Tewari.
\newblock On the universality of online mirror descent.
\newblock In \emph{Advances in neural information processing systems}, pages
  2645--2653, 2011.

\end{thebibliography}
\end{document}